%% file: main.tex
\def\year{2019}\relax
\documentclass[letterpaper]{article} 
\usepackage{aaai19}  
\usepackage{times}  
\usepackage{helvet}  
\usepackage{courier}  
\PassOptionsToPackage{hyphens}{url}\usepackage{hyperref}
\usepackage{graphicx}  
\usepackage{amsthm}
\usepackage[utf8]{inputenc} 
\usepackage[T1]{fontenc}    
\usepackage{hyperref}
\usepackage{booktabs}       
\usepackage{amsfonts}       
\usepackage{nicefrac}       
\usepackage{microtype}      
\usepackage{thm-restate}
\usepackage{algpseudocode,algorithm,algorithmicx}
\usepackage[cal=zapfc,bb=fourier]{mathalfa}
\usepackage{amsmath,amssymb}
\usepackage{rotating}
\usepackage{times}
\usepackage{placeins}
\usepackage{enumitem}   
\usepackage{nicefrac}       
\usepackage{microtype}      
\usepackage[normalem]{ulem} 
\usepackage{adjustbox}
\usepackage{tabulary}

\usepackage{tikz,tikz-cd}
\usetikzlibrary{matrix, calc, arrows}
\usepackage{tabulary,booktabs}
\usetikzlibrary{tikzmark}
\usepackage{stackengine}
\newtheorem{definition}{Definition}
\newtheorem{theorem}{Theorem}

\newtheorem{lemma}{Lemma}

\DeclareMathOperator*{\argmax}{arg\,max}

\def\one{\mbox{1\hspace{-4.25pt}\fontsize{12}{14.4}\selectfont\textrm{1}}}

\usepackage{tikz}
\usepackage{tikz-cd}
\usetikzlibrary{matrix,calc,arrows}
\usetikzlibrary{tikzmark}

\tikzstyle{b} = [rectangle, draw, fill=blue!20, node distance=3cm, text width=6em, text centered, rounded corners, minimum height=4em, thick]
\tikzstyle{c} = [rectangle, draw, inner sep=0.5cm, dashed]
\tikzstyle{l} = [draw, -latex',thick]

\newcommand{\indep}{\rotatebox[origin=c]{90}{$\models$}}

\begin{document}
%
\title{A Formal Approach to Explainability}
\author{Lior Wolf$^{1,2}$~~~~~~~~~~Tomer Galanti$^{2}$~~~~~~~~~~Tamir Hazan$^{3}$\\
~\\
$^1$Facebook AI Research\\
$^2$The School of Computer Science, Tel Aviv University\\
$^3$Technion\\
}
\maketitle
\begin{abstract}
We regard explanations as a blending of the input sample and the model's output and offer a few definitions that capture various desired properties of the function that generates these explanations. We study the links between these properties and between explanation-generating functions and intermediate representations of learned models and are able to show, for example, that if the activations of a given layer are consistent with an explanation, then so do all other subsequent layers. In addition, we study the intersection and union of explanations as a way to construct new explanations.
\end{abstract}

\section{Introduction}

Machine learning is often concerned with tacit knowledge, and tacit knowledge leads to black box models. Given a learned model, one cannot ``crack it open'' in the hope to understand all of the internal nuts and bolts. Explaining the model often relies, instead, on communicating, in a way that is understandable to humans, an internal state of the model during computation. 

An explanation process, therefore, has three components: the input, the model's output for that input, which needs to be justified, and an internal state of the model. The explanation itself combines the input and the output into a joint sample that should be understandable by human users. The explaining function (EF) generates these explanations, based on the two inputs, and is intimately tied to the model it explains. We can expect, therefore, that the generated explanations are linked to internal states of the model.

For example, consider a mapping from images to labels of objects. The explanation often takes the visual form of an image, where the predicted object is highlighted and the features related to the label are emphasized, see, e.g.,~\cite{zeiler2014visualizing}. The algorithmic way to explain, is to generate this hybrid image from the internal representation of the black-box model. 
Another form of explanation is a textual one~\cite{hendricks2016generating}, and describes features that belong to the recognized class. For example, ``[this is an image of a broccoli since] it is green, has a flowering head, and a thick stem with small leaves'', where the part in brackets is the label, but not the explanation. This explanation is both a function of the input image (describes what can be seen and where) and the label (contains known properties of broccolis).

We provide a formal framework that captures various desiderata of explanations, among which are: consistency between an internal model's state and the generated explanation, explainability of an internal state, validity of an explanation, and its completeness.

Our main results link various aspects of the properties. For example, a valid explanation has to be complete. We also study the specific case of explaining, using the gradient of the loss, the predictions of  multiclass neural networks and show that the explanation is linked to the learned representation. Lastly, we study the intersection and unions of explanations, as a way to create new explanations by combining existing ones.

\section{Settings}

We describe a few fundamental concepts in a way that is less formal than what is presented in the subsequent sections. An illustration of the main components of our framework is given in Fig.~\ref{fig:illustration}.

\noindent{\bf What do we want to explain?} Given a function $h:\mathcal{X}\rightarrow \mathcal{Y}$ from the input domain $\mathcal X$ to the output domain $\mathcal{Y}$, we would like to explain the output $h(x)$ for some input $x\in \mathcal{X}$. $h$ is typically a learned model.

\noindent{\bf What is an explanation?} An explanation is a blending of the input and the output. An explanation function (EF for short) $g:\mathcal{X} \times \mathcal {Y}\to G$  maps $x\in \mathcal{X}$ and $y\in \mathcal{Y}$ to $g(x,y)$, which is the explanation for $(x,y)$ in the blended domain $G$.  Hopefully, the elements of domain $G$ are understandable to humans. However, this part is not amendable to formalization.

\noindent{\bf Consistent representation:} Given a function $h:\mathbb{R}^n \to \mathcal{Y}$ of the form $h = c \circ f$, where $f$ is some representation of the input and $c$ a classifier on top of it, we would like to discuss the link between $f$ and an EF $g$. We say that $f$ is consistent with respect to an EF $g$, if for all $x_1,x_2 \in \mathcal{X}$, such that: $\vert g(x_1,h(x_1)) - g(x_2,h(x_2))\vert \leq \epsilon$, we have: $\vert f(x_1) - f(x_2)\vert \leq \beta(\epsilon)$.

\noindent{\bf Explainable representation:} This definition is similar to consistency, with a reversed implication. We say that $f$ is explainable with respect to the EF $g$ if for all $x_1,x_2 \in \mathcal{X}$, such that: $\vert f(x_1) - f(x_2)\vert \leq \epsilon$, we have: $\vert g(x_1,h(x_1)) - g(x_2,h(x_2))\vert \leq \gamma(\epsilon)$.

\noindent{\bf Equivalence between an EF and a representation:} A representation $f$ is equivalent to an EF $g$, if it is both consistent with it and explainable by it.

\noindent{\bf Valid explanation:} An EF $g$ is valid, if there exists a function $t$, such that the model's label is predictable from the explanation $t(g(x,h(x)))\approx h(x)$. 

\noindent{\bf Complete explanation:} We say that an EF $g$ is complete in the context of a model $h$, if there is no information left in the input $x$ that is relevant to $h$, which is independent of the information in $g(x,h(x))$. If we define as $\bar g(x,h(x))$ all the information that is the part of $x$ but which has no information on $g(x,h(x))$, then $g$ is complete if there is no function $s$ such that $s(\bar g(x,h(x)))\approx h(x)$.

\noindent{\bf Intersection  and Union of EFs:} Given a model $h:\mathbb{R}^n \to \mathcal{Y}$ and two EFs $g_1,g_2$, the intersection between them is a representation $u(x,h(x))$ such that we can write $r_1(g_1(x,h(x))) = (e_1(x,h(x)),u(x,h(x)))$ and $r_2(g_2(x,h(x))) = (e_2(x,h(x)),u(x,h(x)))$, where $r_1,r_2$ are invertible transformations and $e_1$ is the part of $g_1$ that is independent of $g_2$ (and vice versa for $e_2$). The union between them is defined as $(e_1(x,h(x)),u(x,h(x)),e_2(x,h(x)))$.


\section{A Formal Model}\label{sec:expl}

In this section, we present our formal model of explainability. The sample space $\mathcal{Z} := \mathcal{X} \times \mathcal{Y}$, where $\mathcal{X} \subset \mathbb{R}^n$ is the inputs space and $\mathcal{Y}$ is the outputs space. For instance, in binary classification, $\mathcal{Y} = \{\pm 1\}$, in multi-class classification $\mathcal{Y} = \{1,\dots,K\} := [K]$ for some $K \in \mathbb{N}$, and in regression, $\mathcal{Y} = \mathbb{R}$. In addition, there is an unknown target function $y:\mathbb{R}^n \to \mathcal{Y}$ that is being learned and a hypothesis class $\mathcal{H}$ of models $h:\mathbb{R}^n \to \mathcal{Y}$ from which the learning algorithm selects an approximation of the target function $y$. We denote by $D$ the distribution of data samples in $\mathcal{X}$.

We consider a family of EFs $\mathcal{G}$, and each EF $g \in \mathcal{G}$ is a mapping $g:\mathbb{R}^n \times \mathcal{Y} \to G$. Here, $G$ is a set of possible explanations. We do not aim to show how to compute an explanation $g(x,h(x))$.  Instead, we focus on providing useful terminology to understand the properties of EFs. 


\input{illustration.tex}

\paragraph{Terminology and notations} 

Before we present our main results, we recall a few technical notations. First, throughout this manuscript, we will assume that $D$ is supported by $\mathcal{X}$, which is assumed, for the purpose of simplifying entropy and mutual-information based arguments, to be a discrete set. We also assume that all logarithms are base $2$. The image of a function $f:\mathcal{X}_1 \to \mathcal{X}_2$ is denoted by $f(\mathcal{X}_1)$. We denote by, $\ell:\mathcal{Z} \to \mathbb{R}$ a loss function. Typically, in binary classification, we have the zero-one loss, $\ell(y_1,y_2) = \one[y_1 \neq y_2]$ and in regression, we often employ the L1 loss $\vert y_1 - y_2\vert_1$ or the L2 loss $\vert y_1 - y_2\vert^2$. Here, $\one[b]$ is an indicator of a boolean variable, $b$, being true, i.e., $\one[\textnormal{true}] = 1$ and $\one[\textnormal{false}] = 0$.

We recall the classical information theoretic notations from~\cite{Cover:2006:EIT:1146355}: the expectation and probability operators symbols $\mathbb{E}, \mathbb{P}$, the Shannon entropy (discrete or continuous) $H(X) := -\mathbb{E}_{X}[ \log \mathbb{P}[X]]$, the conditional entropy $H(X|Y) := H(X,Y) - H(Y)$ and the (conditional) mutual information (discrete or continuous) $I(X; Y| Z) := H(X|Z) - H(X|Y,Z)$. For a given value $p \in [0,1]$, we denote, $H(p) = -p\log(p)-(1-p)\log(1-p)$.

\subsection{Properties of EFs}

We provide formal definitions to the various properties mentioned in the Settings Section. A representation of the input is a function $f:\mathcal{X} \to \mathbb{R}^d$ (for some $d>0$). In most cases, we will assume that $f$ is a sub-architecture of our mapping $h:\mathbb{R}^n \to \mathcal{Y}$. Specifically, we would consider $h$ to be a composite function that is built in layers $h = p_k \circ \dots \circ  p_1$, where each layer $p_i$ is a function $p_i:\mathbb{R}^{n^{i-1}} \to \mathbb{R}^{n^{i}}$ (for some $k,n^i\in \mathbb{N}$, $n^{0}$ being the input dimension $n$ and $i \in \{1,2\dots,k\}$). In this case, $f$ would contain the first $m$ layers $f=p_m \circ \dots \circ p_1$ and $c$ would contain the $k-m$ top layers: $c=p_{k} \circ \dots \circ p_{m+1}$.

\begin{definition}[Consistent Representation] Let $h = c \circ f \in \mathcal{H}$ be a model, $g:\mathcal{Z} \to G$ an EF and $\beta:(0,\infty) \to [0,\infty)$. We say that $f$ is a $\beta(\epsilon)$-consistent representation with respect to $g$, if for any $\epsilon \in (0,\infty)$ and $x_1,x_2 \in \mathcal{X}$, we have:
\begin{equation}
\begin{aligned}
&\vert g(x_1,h(x_1)) - g(x_2,h(x_2)) \vert \leq \epsilon \\
\implies& \vert f(x_1) - f(x_2) \vert \leq \beta(\epsilon)  
\end{aligned}
\end{equation}
\end{definition}

\begin{definition}[Explainable Representation] Let $h = c \circ f \in \mathcal{H}$ be a model and $g:\mathcal{Z} \to G$ an EF. For a given function $\gamma: (0,\infty) \to (0,\infty)$, we say that $f$ is a $\gamma(\epsilon)$-explainable representation with respect to $g$, if for any $\epsilon \in (0,\infty)$ and $x_1,x_2 \in \mathcal{X}$, we have:
\begin{equation}
\begin{aligned}
& \vert f(x_1) - f(x_2) \vert \leq \epsilon \\
\implies& \vert g(x_1,h(x_1)) - g(x_2,h(x_2)) \vert \leq \gamma(\epsilon)  
\end{aligned}
\end{equation}
Additionally, for a given function $\gamma: (0,\infty)\times (0,\infty) \to (0,\infty)$, we say that $f$ is second-order $\gamma(\epsilon_0,\epsilon_1)$-explainable with respect to $g$, if for any $\epsilon_0,\epsilon_1 \in (0,\infty)$ and $x_1,x_2 \in \mathcal{X}$, we have:
\begin{equation}
\begin{aligned}
& \vert f(x_1) - f(x_2) \vert \leq \epsilon_0 \textnormal{ and } \Big\vert \frac{\partial f(x_1)}{\partial x_1} - \frac{\partial f(x_2)}{\partial x_2} \Big\vert \leq \epsilon_1\\
\implies& \vert g(x_1,h(x_1)) - g(x_2,h(x_2)) \vert \leq \gamma(\epsilon_0,\epsilon_1)  
\end{aligned}
\end{equation}
\end{definition}

\begin{definition}[Equivalence between a Representation and an EF] Let $h = c \circ f \in \mathcal{H}$ be a model, $g:\mathcal{Z} \to G$ an EF and $\beta,\gamma:(0,\infty) \to [0,\infty)$. We say that $f$ is $(\beta(\epsilon),\gamma(\epsilon))$-equivalent to $g$, if it is $\beta(\epsilon)$-consistent and $\gamma(\epsilon)$-explainable with respect to $g$.
\end{definition}

\begin{definition}[Valid EF] Let $h \in \mathcal{H}$ be a model, $g:\mathcal{Z} \to G$ an EF, $\epsilon_0>0$ a fixed constant and $x\sim D$. We say that $g$ is $\epsilon_0$-valid with respect to $h$, if there is a function $t:G \to \mathcal{Y}$ that satisfies: 
\begin{equation}\label{eq:valid}
\mathbb{E}_{x}[\ell(t(g(x,h(x))),h(x))] \leq \epsilon_0
\end{equation}
\end{definition}

\begin{definition}[Complete EF] Let $h \in \mathcal{H}$ be a model, $g:\mathcal{Z} \to G$ an EF and $x\sim D$. Let $\alpha,\epsilon>0$ be two constants. We say that $g$ is $(\epsilon,\alpha)$-complete with respect to $h$, if every function $\bar g:\mathcal{X} \to \mathbb{R}^d$, such that, $I(g(x,h(x));\bar{g}(x)) \leq \epsilon$ and function $s:\mathbb{R}^d \to \mathcal{Y}$, we have:
\begin{equation}
\mathbb{E}_{x}[\ell(s(\bar g(x)),h(x))] \geq \alpha
\end{equation}
\end{definition}

\section{Linking Representations and EFs}

The following theorem states that if an internal representation of a layered model $h$ is $\beta(\epsilon)$-consistent with an EF, then, under mild conditions, downstream layers are also consistent with the specified EF.

\begin{theorem}\label{thm:consistent} Let $h = p_k \circ \dots \circ p_1:\mathbb{R}^n \to \mathcal{Y}$ be a model and $g:\mathcal{Z} \to G$ an EF. Assume that $f_i := p_i \circ \dots \circ p_1$ is $\beta(\epsilon)$-consistent with respect to $g$, for some $i \in \{1,\dots,k\}$. Assume that $p_r$ is a $l_r$-Lipschitz function for every $r \in \{i+1,\dots,j\}$. Then, $f_j := p_j \circ \dots \circ p_1$ is $\hat{\beta}(\epsilon)$-consistent with respect to $g$, for $\hat{\beta}(\epsilon) := \beta(\epsilon) \cdot \prod^{j}_{r=i+1} l_r$. 
\end{theorem}

\begin{proof} Assume that $f_i$ is $\beta(\epsilon)$-consistent for some $i \in \{1,\dots,k\}$. Let $x_1,x_2 \in \mathcal{X}$ be two inputs, such that, $\vert g(x_1,h(x_1)) - g(x_2,h(x_2)) \vert \leq \epsilon$. Then, for every $j \in \{i,\dots,k\}$, we have: 
\begin{equation}
\begin{aligned}
&\vert f_j(x_1) - f_j(x_2)\vert \\
=& \vert p_j\circ \dots \circ p_{i+1} \circ f_i(x_1) - p_j\circ \dots \circ p_{i+1} \circ f_i(x_2)\vert \\
\leq& \prod^{j}_{r=i+1} l_{r} \vert f_i(x_1) - f_i(x_2)\vert \leq \beta(\epsilon) \cdot \prod^{j}_{r=i+1} l_{r} = \hat{\beta}(\epsilon)
\end{aligned}
\end{equation}
Since each $p_r$ is a $l_r$-Lipschitz continuous function for every $r \in \{i+1,\dots,j\}$.
\end{proof}

One implication of this result is that if a layer of a neural network model $h$ is consistent with an explanation $g$, then $h$ itself is also consistent, i.e., in the case where any of the layers of $h$ is consistent with $g$, then if $g(x,h(x))$, which is a function of $h(x)$ as well as of $x$, does not change much when replacing $x$ with $x'$, then $h(x)$ and $h(x')$ are similar.

The following theorem deals with upstream layers: under mild assumptions, if $f$ is an explainable representation, that is obtained as a layer of a neural network model $h$, then so are the previous layers in this network.

\begin{theorem}\label{thm:explainable} Let $h = p_k \circ \dots \circ p_1:\mathbb{R}^n \to \mathcal{Y}$ be a model and $g:\mathcal{Z} \to G$ an EF. Assume that $f_i := p_i \circ \dots \circ p_1$ is $\gamma(\epsilon)$-explainable with respect to $g$, for some $i \in \{1,\dots,k\}$. Assume that $p_r$ is a $l_r$-Lipschitz function for every $r \in \{j+1,\dots,i\}$. Then, $f_j := p_j \circ \dots \circ p_1$ is $\hat{\gamma}(\epsilon)$-explainable with respect to $g$, for $\hat{\gamma}(\epsilon) := \gamma\left(\epsilon \cdot \prod^{i}_{r=j+1} l_r \right)$. 
\end{theorem}

\begin{proof} Assume that $f_i$ is $\gamma(\epsilon)$-explainable for some $i \in  \{1,\dots,k\}$. Let $x_1,x_2 \in \mathcal{X}$ be two inputs, such that, $\vert f_j(x_1) - f_j(x_2) \vert \leq \epsilon$. Then,
\begin{equation}
\begin{aligned}
&\vert f_i(x_1) - f_i(x_2)\vert \\
=& \vert p_i\circ \dots \circ p_{j+1} \circ f_j(x_1) - p_i\circ \dots \circ p_{j+1} \circ f_j(x_2)\vert \\
\leq& \prod^{i}_{r=j+1} l_{r} \vert f_j(x_1) - f_j(x_2)\vert \leq \epsilon \cdot \prod^{i}_{r=j+1} l_{r}
\end{aligned}
\end{equation}
Since each $p_r$ is a $l_r$-Lipschitz continuous function for every $r \in \{j+1,\dots,i\}$. Therefore, since $f_i$ is $\gamma(\epsilon)$-explainable with respect to $g$, we have:
\begin{equation}
\begin{aligned}
\vert g(x_1,h(x_1)) - g(x_2,h(x_2))\vert \leq \gamma\left(\epsilon \cdot \prod^{i}_{r=j+1} l_{r} \right)
\end{aligned}
\end{equation}
\end{proof}

Note that an immediate implication is that if a representation is explainable by $g$, then so is the input $x$ itself.

\subsection{A Specific Case Study}

We next treat a specific case, which is the conventional multiclass classification approach for deep neural networks,  coupled with the iconic image-based explanation that is given by the derivative of the output neuron associated with the predicted label by the input. In this case, the model predicts the label based on an $\argmax$ of multiple 1D linear projections ($m_i$, $i$ being the index of the label) of the activations of the penultimate layer $p(x)$ for some input $x$. The explanation of the prediction $h(x)$ is then given as the matrix derivative of $(m^\top_{h(x)} \cdot p(x))$ by the input $x$.

The following theorem states that if our model is of the form $h(x) = \argmax_{i\in \mathcal{Y}} (m^\top_{i} \cdot p(x))$ and our EF has the form $g(x,h(x)) = \frac{\partial (m^\top_{h(x)} \cdot p(x))}{\partial x}$, where $p = c \circ f$ such that $c$, $f$ and the derivative of $c$ are Lipschitz continuous functions, then, $f$ is explainable with respect to $g$.

\begin{theorem}\label{thm:dervG} Let $\mathcal{Y} = [K]$ and $h:\mathbb{R}^n \to \mathcal{Y}$ a model of the form, $h(x) = \argmax_{i\in \mathcal{Y}} m^\top_i \cdot p(x)$, where $p:\mathbb{R}^n \to \mathbb{R}^d$ and $m_i \in \mathbb{R}^d$, for $i \in [K]$. Let $g(x,h(x)) = \frac{\partial (m^\top_{h(x)} \cdot p(x))}{\partial x}$ be an EF. Assume that for all $i \in [K]$, $p = c \circ f$, such that: $c$, $\frac{\partial c(x)}{\partial x}$, $\frac{\partial p(x)}{\partial x}$ and $f$ are Lipschitz continuous functions. Additionally, assume that: $\forall i\neq j \in [K], x \in \mathcal{X}: m^{\top}_i \neq m^{\top}_j $ and $\forall x \in \mathcal{X}: \vert p(x)\vert \geq \Delta$, for some constant $\Delta >0$. Then, $f$ is second-order $\mathcal{O}(\epsilon_0+\epsilon_1)$-explainable with respect to $g$.
\end{theorem}

\begin{proof} Assume that for all $i \in [K]$:
\begin{equation}
\vert f(x_1) - f(x_2)\vert \leq \epsilon_0 \textnormal{ and } \Big\vert \frac{f(x_1)}{\partial x_1} - \frac{f(x_2)}{\partial x_2} \Big\vert \leq \epsilon_1 
\end{equation}
Then, since each $c$ is a Lipschitz continuous function, there is a constant $l_1,\dots,l_K>0$, such that for all $i \in [K]$ and $x_1,x_2 \in \mathbb{R}^n$:
\begin{equation}
\begin{aligned}
&\vert m^\top_i \cdot p(x_1) - m^\top_i \cdot p(x_2) \vert \\
=& \vert m^{\top}_i \vert \cdot \vert p(x_1) - p(x_2) \vert \leq l \cdot \vert m^{\top}_i \vert \cdot \epsilon_0 
\end{aligned}
\end{equation}
For any small enough $\epsilon_0 > 0$, we have: 
\begin{equation}
l \cdot \max_{i\in [K]} \vert m^{\top}_i \vert \cdot \epsilon_0 < \min_{i\neq j}\vert m^{\top}_i - m^{\top}_j \vert \cdot \Delta/2
\end{equation}
Since $\forall i\neq j \in [K],x\in \mathcal{X}: \vert p(x)\vert \geq \Delta$, we have:
\begin{equation}
\vert m^{\top}_i \cdot p(x) - m^{\top}_j \cdot p(x)\vert \geq  \min_{i\neq j}\vert m^{\top}_i - m^{\top}_j \vert\cdot \Delta > 0
\end{equation}
In this case, if $h(x_1) = i$, then, for all $j \in [K]$, such that $j \neq i$, we have:
\begin{equation}
\begin{aligned}
&m^{\top}_i \cdot p(x_2) - m^{\top}_j \cdot p(x_2) \\
\geq& m^{\top}_i \cdot p(x_1) - \vert m^{\top}_i \cdot p(x_1) - m^{\top}_i \cdot p(x_2) \vert \\
&- m^{\top}_j \cdot p(x_2) - \vert m^{\top}_j \cdot p(x_1) - m^{\top}_j \cdot p(x_2) \vert \\
\geq& \min_{i\neq j}\vert m^{\top}_i - m^{\top}_j \vert\cdot \Delta - 2l \cdot \max_{i\in [K]} \vert m^{\top}_i \vert \cdot \epsilon_0 > 0
\end{aligned}
\end{equation}
Therefore, we conclude that: $h(x_1) = h(x_2) = i$. Thus,
\begin{equation}
\begin{aligned}
&\vert g(x_1,h(x_1)) -g(x_2,h(x_2)) \vert \\
=&\Big\vert \frac{\partial (m^\top_{h(x_1)} \cdot p(x_1))}{\partial x_1} - \frac{\partial (m^\top_{h(x_2)} \cdot p(x_2))}{\partial x_2} \Big\vert \\
=&\Big\vert \frac{\partial (m^\top_{i} \cdot p(x_1))}{\partial x_1} - \frac{\partial (m^\top_{i} \cdot p(x_2))}{\partial x_2} \Big\vert \\
=& \vert m^{\top}_i \vert \cdot \Big\vert \frac{\partial p(x_1)}{\partial x_1} - \frac{\partial p(x_2)}{\partial x_2} \Big\vert  \\
=& \Big\vert \frac{\partial c(f(x_1))}{\partial f(x_1)} \cdot \frac{\partial f(x_1)}{\partial x_1} - \frac{\partial c(f(x_2))}{\partial f(x_2)} \cdot \frac{\partial f(x_2)}{\partial x_2} \Big\vert \\
\end{aligned}
\end{equation}
Since $c$ and $f$ are Lipschitz continuous functions, we have:
\begin{equation}
\begin{aligned}
&\vert g(x_1,h(x_1)) -g(x_2,h(x_2)) \vert \\
=& \mathcal{O}\left(\Big\vert \frac{\partial f(x_1)}{\partial x_1} - \frac{\partial f(x_2)}{\partial x_2} \Big\vert \right) \\
&+ \mathcal{O}\left(\Big\vert \frac{\partial c(f(x_1))}{\partial f(x_1)} - \frac{\partial c(f(x_2))}{\partial f(x_2)} \Big\vert \right) \\
\end{aligned}
\end{equation}
Since $\frac{\partial c(u)}{\partial u}$ is also a Lipschitz continuous function, we have:
\begin{equation}
\begin{aligned}
&\vert g(x_1,h(x_1)) -g(x_2,h(x_2)) \vert\\
=& \mathcal{O}(\epsilon_1 + \vert f_i(x_1) - f_i(x_2)\vert) = \mathcal{O}(\epsilon_0+\epsilon_1) \\
\end{aligned}
\end{equation}
\end{proof}

\section{Validity and Completeness}

The next result shows that if an EF is valid, then it is also complete. {\color{black} The intuition behind this result is, if we are able to recover $h(x)$ from $\bar{g}(x)$ and from $g(x,h(x))$, then, $\bar{g}(x)$ and $g(x,h(x))$ cannot be independent of each other.   }

\begin{theorem}[Valid $\implies$ Complete] Let $h:\mathbb{R}^n \to \mathcal{Y}$ be a model, $g:\mathcal{Z} \to G$ an $\epsilon_0$-valid EF for some constant $\epsilon_0 \in (0,0.5)$ and $x\sim D$. Assume that $\mathcal{Y} = \{\pm 1\}$ and denote, $p := \mathbb{P}[h(x)=1]$. Then, $g$ is $(\epsilon,\alpha)$-complete with respect to $h$, with $\alpha := \frac{\sqrt{1+H(p)(H(p) - \epsilon - 2\sqrt{\epsilon_0})} - 1}{H(p)}$ and any $\epsilon>0$ that satisfies, $H(p) > \epsilon + 2\sqrt{\epsilon_0}$.  In particular, if $p=1/2$, we have: $\alpha = \sqrt{2-\epsilon-2\sqrt{\epsilon_0}}-1$.
\label{thm:validcomplete}
\end{theorem}
\begin{proof} Let $\bar{g}:\mathcal{X} \to \mathbb{R}^d$ be a function, such that, $I(\bar{g}(x);g(x,h(x))) \leq \epsilon$. Since $g(x,h(x))$ is $\epsilon_0$-valid, there is a function $t:G \to \mathcal{Y}$, that satisfies: 
\begin{equation}
\begin{aligned}
& \mathbb{P}[t(g(x,h(x)))\neq h(x)] \\
=& \mathbb{E}_{x}[\ell(t(g(x,h(x))),h(x))] \leq \epsilon_0 < 1/2
\end{aligned}
\end{equation}
By $I(X;f(Y)) \leq I(X;Y)$, for every function $f$, we have:
\begin{equation}
\begin{aligned}
I(\bar{g}(x);t(g(x,h(x)))) \leq& I(\bar{g}(x);g(x,h(x))) \leq \epsilon
\end{aligned}
\end{equation} 
By Lem.~\ref{lem:binaryNEQ} in the Appendix,
\begin{equation}
\begin{aligned}
&\vert I(\bar{g}(x);h(x)) - I(\bar{g}(x);t(g(x,h(x))))\vert  \\
\leq& H(\mathbb{P}[t(g(x,h(x))) \neq h(x)])
\end{aligned}
\end{equation}
Therefore, by Lem.~\ref{lem:entropyUpper}, we have:
\begin{equation}
\begin{aligned}
I(\bar{g}(x);h(x)) &\leq \epsilon + H(\mathbb{P}[t(g(x,h(x))) \neq h(x)]) \\
&\leq \epsilon + 2\sqrt{\epsilon_0}
\end{aligned}
\end{equation}
Next, we assume that $H(q) > \epsilon + 2\sqrt{\epsilon_0}$, where $p = \mathbb{P}[h(x) = 1]$. Let $\alpha := \frac{\sqrt{1+H(p)(H(p) - \epsilon - 2\sqrt{\epsilon_0})} - 1}{H(p)}$ and assume by way of  contradiction that there is a function $s:\mathbb{R}^{d} \to \mathcal{Y}$, that satisfies: 
$\mathbb{E}_{x}[\ell(s(\bar{g}(x)),h(x))] < \alpha$. 
Then, by Lem.~\ref{lem:fanoregev} in the Appendix, we have:
\begin{equation}
\begin{aligned}
I(\bar{g}(x);h(x)) &> (1-\alpha) H(p) - H(\alpha) \\
&\geq (1-\alpha) H(p) - 2\sqrt{\alpha}\\
\end{aligned}
\end{equation}
We conclude that:
\begin{equation}
(1-\alpha) H(p) - 2\sqrt{\alpha} <  \epsilon + 2\sqrt{\epsilon_0}
\end{equation}
finally, by the quadratic formula, we arrive at a contradiction for $\alpha = \frac{\sqrt{1+H(p)(H(p) - \epsilon - 2\sqrt{\epsilon_0})} - 1}{H(p)}$. Therefore, we conclude that, $g(x,h(x))$ is $\epsilon$-complete.
\end{proof}

\section{EF Operators}

We next study the arithmetic of explanations. The practical utility of this is left for future research. However, we can imagine that by combining elementary explanations to complex ones and by intersecting these complex explanations, one can algorithmically construct explanations.

\begin{definition}[Intersection and Union of Random Variables] Let $x \sim D$ and $f_1:\mathcal{X} \to \mathcal{X}_1$ and $f_2:\mathcal{X} \to \mathcal{X}_2$ are two functions. We say that the random variables $f_1(x)$ and $f_2(x)$ $\epsilon$-intersect, if there are two invertible functions $r_1:\mathcal{X}_1 \to \mathcal{V}_1$ and $r_2:\mathcal{X}_2 \to \mathcal{V}_2$, such that, $r_1(f_1(x)) = (e_1(x),u(x))$ and $r_2(f_2(x)) = (e_2(x),u(x))$, where $I(e_i(x); f_j(x)) \leq \epsilon$ (for any $i\neq j \in \{1,2\}$). We call the random variable $u(x)$, the $\epsilon$-intersection of $f_1(x)$ and $f_2(x)$. In addition, we call $(e_1(x),u(x),e_2(x))$ the $\epsilon$-union of $f_1(x)$ and $f_2(x)$.
\end{definition}

By Lem.~\ref{lem:unique} in the appendix, the intersection and union of two random variables $f_1(x)$ and $f_2(x)$ are unique, up to invertible transformations. 

The following results show that the intersection of two EFs, one of which is valid and the other complete, is a valid EF.

\begin{theorem}\label{thm:intersection} Let $h:\mathbb{R}^n \to \mathcal{Y}$ be a model, $g_1,g_2:\mathcal{Z} \to G$ two EFs and $\epsilon,\epsilon_0,\alpha>0$ three constants. Assume that $\mathcal{Y} = \{\pm 1\}$, $g_1(x,h(x))$ and $g_2(x,h(x))$ $\epsilon$-intersect and denote by $u(x,h(x))$ the $\epsilon$-intersection of them. Assume that $g_1$ is $\epsilon_0$-valid (w.r.t $h$) and $g_2$ is $(\epsilon,\alpha)$-complete (w.r.t $h$). Then, $u$ is $\epsilon_1$-valid (w.r.t $h$), for $\epsilon_1 :=1 - \frac{2^{- \epsilon_0 - 2\sqrt{\epsilon_0} - H(h(x))}}{1-\alpha}$.
\end{theorem}

\begin{proof} Let $r_1:G \to \mathcal{V}_1$ and $r_2:G \to \mathcal{V}_2$ be two invertible functions, such that, $r_1(g_1(x,h(x))) = (e_1(x,h(x)),u(x,h(x)))$ and $r_2(g_2(x,h(x))) = (e_2(x,h(x)),u(x,h(x)))$, where, $e_i(x,h(x)) \indep g_j(x,h(x))$ (for any $i\neq j \in \{1,2\}$).  By the chain rule property of mutual information,
\begin{equation}
\begin{aligned}
&I(e_1(x,h(x)),u(x,h(x));h(x)) \\ 
=& I(e_1(x,h(x));h(x)) \\
&+ I(u(x,h(x));h(x)\vert e_1(x,h(x)))\\
\leq& I(e_1(x,h(x));h(x)) + I(u(x,h(x));h(x)) \\
\end{aligned}
\end{equation}
Therefore, we have:
\begin{equation}
\begin{aligned}
&I(u(x,h(x));h(x))\\ 
\geq & I(e_1(x,h(x)),u(x,h(x));h(x)) \\
&- I(e_1(x,h(x));h(x))\\
= & I(r_1(e_1(x,h(x)),u(x,h(x)));h(x)) \\
&- I(e_1(x,h(x));h(x))\\
= & I(g_1(x,h(x));h(x)) - I(e_1(x,h(x));h(x))\\
\end{aligned}
\end{equation}
Since $g_1$ is $\epsilon_0$-valid, there is a function, $t_1:G \to \mathcal{Y}$, such that:
\begin{equation}
\begin{aligned}
& \mathbb{P}_{x\sim D}[t_1(g_1(x,h(x)))\neq h(x)] \\
=& \mathbb{E}_{x\sim D}[\ell(t_1(g_1(x,h(x))),h(x))] \leq \epsilon_0 < 1/2
\end{aligned}
\end{equation}
Therefore, by Lem.~\ref{lem:fanoregev} and Lem.~\ref{lem:entropyUpper} in the Appendix, we have:
\begin{equation}
\begin{aligned}
&I(g_1(x,h(x));h(x)) \\
\geq& (1-\epsilon_0) H(h(x)) - H(1-\epsilon_0) \\
\geq& (1-\epsilon_0) H(h(x)) - 2\sqrt{\epsilon_0}
\end{aligned}
\end{equation}
By the definition of $e_1(x,h(x))$, we have:
\begin{equation}
I(e_1(x,h(x));g_2(x,h(x)))\leq \epsilon
\end{equation}
Therefore, since $g_2$ is $(\epsilon,\alpha)$-complete, for every function $s$ with outputs in $\mathcal{Y}$, we have:
\begin{equation}
\mathbb{E}_{x\sim D}[\ell(s(e_1(x,h(x))),h(x))] \geq \alpha
\end{equation}
Therefore, by Lem.~\ref{lem:fanoconverse} in the Appendix,
\begin{equation}
I(e_1(x,h(x));h(x)) \leq \log(1-\alpha) + H(h(x))
\end{equation}
We conclude that:
\begin{equation}
\begin{aligned}
&I(u(x,h(x));h(x))\\ 
\geq& (1-\epsilon_0) H(h(x)) - 2\sqrt{\epsilon_0} \\
& - (\log(1-\alpha) + H(h(x)))\\
\geq & \log\left(\frac{1}{1-\alpha}\right) - \epsilon_0\cdot H(h(x)) -2\sqrt{\epsilon_0}  \\
\geq & \log\left(\frac{1}{1-\alpha}\right) - \epsilon_0 -2\sqrt{\epsilon_0}  \\
\end{aligned}
\end{equation}
Finally, by Lem.~\ref{lem:fanoconverse} in the Appendix, there is a function $t_2$ with outputs in $\mathcal{Y}$, such that:
\begin{equation}
\begin{aligned}
&\mathbb{E}_{x\sim D}[\ell(t_2(u(x,h(x))),h(x))] \\
\leq& 1 - \frac{2^{- \epsilon_0 - 2\sqrt{\epsilon_0} - H(h(x))}}{1-\alpha}
\end{aligned}
\end{equation}
\end{proof}

Similar results hold for the union of two EFs: if at least one of which is valid, the union is a valid EF, and a similar result for at least one complete EF. 

\begin{lemma}\label{lem:union} Let $h:\mathbb{R}^n \to \mathcal{Y}$ be a model, $g_1,g_2:\mathcal{Z} \to G$ two EFs and $\epsilon,\epsilon_0,\alpha>0$ three constants. Assume that $\mathcal{Y} = \{\pm 1\}$, $g_1(x,h(x))$ and $g_2(x,h(x))$ $\epsilon$-intersect and denote by $\hat{g}(x,h(x))$ the $\epsilon$-union of them. If $g_1$ (or $g_2$) is $\epsilon_0$-valid (w.r.t $h$), then, $\hat{g}$ is $\epsilon_0$-valid as well. Additionally, if $g_1$ (or $g_2$) is $(\epsilon_1,\alpha)$-complete (w.r.t $h$), $\hat{g}$ is also $(\epsilon_1,\alpha)$-complete.
\end{lemma}

\begin{proof} First, by the definition of $\epsilon$-union, there is a representation, $\hat{g}(x,h(x)) = (e_1(x,h(x)), u(x,h(x)), e_2(x,h(x)) )$, such that, there is an invertible function $r$, that satisfies: $r(e_1(x,h(x)),u(x,h(x))) = g_1(x,h(x))$.

We would like to prove that if $g_1$ is $\epsilon_0$-valid, then, $\hat{g}$ is also $\epsilon_0$-valid. Since, $g_1$ is $\epsilon_0$-valid, there is a function $t:G \to \mathcal{Y}$, such that:
\begin{equation}
\mathbb{E}_{x\sim D}[\ell(t(g_1(x,h(x))),h(x))] \leq \epsilon_0
\end{equation}
In addition, by the definition of $\hat{g}$, we have a representation: $\hat{g}(x,h(x)) = (e_1(x,h(x)), u(x,h(x)), e_2(x,h(x)) )$, such that, there is an invertible function $r$, that satisfies: $r(e_1(x,h(x)),u(x,h(x))) = g_1(x,h(x))$. Therefore, we define, $r'(\hat{g}(x,h(x))) = g_1(x,h(x))$ and obtain, 
\begin{equation}
\mathbb{E}_{x\sim D}[\ell(t\circ r'(\hat{g}(x,h(x))),h(x))] \leq \epsilon_0
\end{equation}
Hence, $\hat{g}$ is also $\epsilon_0$-valid.

Next, we prove that if $g_1$ is $(\epsilon_1,\alpha)$-complete, then, $\hat{g}$ is also $(\epsilon_1,\alpha)$-complete. Let $\bar{g}(x)$ be a function that satisfies: $I(\bar{g}(x);\hat{g}(x,h(x))) \leq \epsilon_1$. In particular, there is a representation 
\begin{equation}
\begin{aligned}
&I(\hat{g}(x);\hat{g}(x,h(x))) \\
=& I(\hat{g}(x);e_1(x,h(x)), u(x,h(x)), e_2(x,h(x))) \\
\geq& I(\hat{g}(x);e_1(x,h(x)), u(x,h(x))) \\
=& I(\hat{g}(x);r(e_1(x,h(x)), u(x,h(x)))) \\
=& I(\hat{g}(x);g_1(x,h(x))) \\
\end{aligned}
\end{equation}
Therefore, $I(\hat{g}(x);g_1(x,h(x))) \leq \epsilon_1$. Since, $g_1$ is $(\epsilon_1,\alpha)$-complete, for any function $s$, we have:
\begin{equation}
\mathbb{E}_{x\sim D}[\ell(s(\bar g(x)),h(x))] \geq \alpha
\end{equation}
In particular, we conclude that $\hat{g}$ is also $(\epsilon_1,\alpha)$-complete.
\end{proof}

\section{Discussion}

In this work, we have studied the properties of EFs $g$. We do not propose new ways to obtain such $g$, which is an active research topic with an increasing interest. Our focus is on blending functions, which mix the input and the output. We view this is a basic property of a wide class of existing and future types of explanations.

The challenge in formalizing EFs using conventional machine learning tools, is that these are not learned from data (they are designed by the practitioners). Therefore, one cannot use the usual convergence-based results. The claims that can be made are based on the mutual information between the model and the EF, the structure of the EF as a two-input function, and the validity requirement, which entails a specific recursive formula $h\approx t(g(x,h(x)))$.

There are three levels of abstractions, which are often referred to as explanations. One is the concrete explanation itself, which for us is an object in domain $G$, which is the target domain of $g$. The second one is the function that generates such explanations. We call these EFs. The third level is the algorithm that provides the EF $g$ given a model $h$. Our analysis focuses on the EF level and it is important to note that $g$ is not general to all $h$, but is given and analyzed in the context of a specific $h$.

\section{Related Work}

The examples that we have provided on available work on explainable solutions, are a fraction of the growing literature on the subject. See~\cite{Guidotti:2018:SME:3271482.3236009} for a survey. Our work covers what is referred to in this survey as the {\em outcome explanation problem}. It is interesting to contrast the definition of this term, given as Def 4.2 in that survey, to our terminology.

Their definition assumes that the explanation is viewed through the lens of a local model $c_l$, which is constructed by some process $f$ from the black box model ($h$ in our terminology $b$ in theirs) at a specific location $x$. The explanation itself $E(c_l,x)$ maps this local model and the input $x$ to a human interpretable domain. 

The example given is of a decision tree, with decision rules that are based on single attribute values (coordinates of $x$), that approximated the black box model in a given neighborhood of $x$. The explanation is given by the sequence of decisions along the path in this decision tree taken for sample $x$. The well known LIME approach~\cite{lime} also fits this definition well. In this approach, random samples are created in the vicinity of $x$, by perturbing this sample, and are weighed by their distance from $x$, when learning the local model $c_l$.

Our framework does not discuss the process $f(h,x)$. The two frameworks are compatible in the sense that $g(x,h(x))$ can be written as $E(f(h,x),x)$, since our $g$ is a function of $h$ (recall that $g$ is specific for a given $h$), and since $g$ could be a function that is based on local approximations of $h$. However, our framework emphasizes the blending properties of the explanation domain, while their definitions emphasize locality and local proxies of $h$ by simple functions that are easy to explain, such as decision trees or linear functions. 

The notion of locality is deferred in our model to the notions of consistency and explainability. However, these exist between intermediate representations and the EF, and Lipschitz continuity type properties and does not necessarily imply an actual approximation.

Recently, ~\cite{alvarez2018} have suggested a framework to learn models that are explainable by design. The basic structure is of a model that, similar to linear functions, is monotonic and additive in each of a set of learned attributes, and on learning attributes that are meaningful. The explanation itself takes the form of presenting the contribution of each attribute, while explaining the attributes using prototypes. While our framework focuses on explaining general models $h$ and not learning self-explainable models, it is interesting to compare their stated desiderata with ours. 

The specified desiderata on that work are:
\begin{enumerate}
    \item Fidelity: the explanation of $x$ should present the relevant information. This is captured by our validity property (relevancy to the label), as well as by the completeness property. 
    \item Diversity: the attributes should be disentangled and there should not be too many of them. This is a property on the explanation domain $G$, which in their work is also used for the representation of the network's penultimate layer. We consider a broader class of explanations, and our analysis of representations refers to $f$ that can be any layer of the network $h$. 
\item Grounding: the attributes of the explanations should be immediately interpretable to humans. In their model, the interpretation is done through prototype samples. A prototype based $G$ is compatible with our framework. However, we cannot formalize the notion of interpretability. 
\end{enumerate}

\section{Conclusions}

The basic concepts of explanations in AI are elusive for several reasons. First, as mentioned, they need to be interpretable by humans, and human understanding has not been fully modeled. Second, there are multiple approaches in the literature. Third, tacit knowledge, by definition, cannot be fully laid down as a set of rules. 

We build a formal framework for explainable AI, by considering, as a first principle, that outcome explanations blend the input and the prediction. Then, we link representations, which we typically take as intermediate activations of neural network models, to these explanations. The interrelationships between the explanations, the models, and the representations are potent enough to lead to several theoretical results. 

One result is that desirable links between explanations and layers of a neural network cannot be specific to this layer, but also manifest to other layers. Another is that a valid explanation must also be complete. A third result studies explainability in the context of a concrete explanation of the predictions of multiclass neural networks. Lastly, we show results on the union and intersection of explanations.

\section{Acknowledgements}

This project has received funding from the European Research Council (ERC) under the European
Union’s Horizon 2020 research and innovation programme (grant ERC CoG 725974). The contribution of Tomer Galanti is part of Ph.D. thesis research conducted
at Tel Aviv University.

\bibliographystyle{aaai}
\bibliography{gans}

\appendix

\section{Useful Lemmas}

For completeness, we provide some useful lemmas that are being employed in the proofs on the theorems in our paper.

\begin{lemma}\label{lem:fanoregev} Let $X$ and $Y$ be two random variables. Assume that there is a function $F$, such that $\mathbb{P}[F(Y) = X] \geq q \geq 1/2$. Then, $I(X;Y) \geq qH(X) - H(q)$.  
\end{lemma}

\proof{The lemma is a modification of Claim~2.1 in~\cite{regev}.}

\begin{lemma}\label{lem:binaryNEQ} Let $X$, $Y$ and $Z$ be three random variables, where $Y$ and $Z$ are binary. We have:
\begin{equation}
\vert I(X;Y) - I(X;Z)\vert \leq H(\mathbb{P}[Y \neq Z]) 
\end{equation}
\end{lemma}

\proof{See.~\cite{1432228}.}

\begin{lemma}\label{lem:entropyUpper} Let $p \in [0,1]$. Then, 
\begin{equation}
H(p) \leq 2\sqrt{p(1-p)}
\end{equation}
\end{lemma}

\proof{See~\cite{14322289}.}

\begin{lemma}\label{lem:fanoconverse} Let $X$ and $Y$ be two discrete random variables taking values from $\mathcal{S}_1$ and $\mathcal{S}_2$ (resp.). Then, there is a function $t:\mathcal{S}_2 \to \mathcal{S}_2$, such that:
\begin{equation}
\mathbb{P}_{X,Y}[X \neq t(Y)] \leq 1-2^{I(X;Y) - H(X)}    
\end{equation}
\end{lemma}

\proof{See~\cite{Feder1994RelationsBE}.}

\begin{lemma}[Intersection Equivalence]\label{lem:unique}Let $x \sim D$ and $f_1:\mathcal{X} \to \mathcal{X}_1$ and $f_2:\mathcal{X} \to \mathcal{X}_2$ are two functions. In addition, let $u_1(x)$ and $u_2(x)$ be two $\epsilon$-intersections of $f_1(x)$ and $f_2(x)$, i.e., there are two pairs of invertible functions $r^{i}_1:\mathcal{X}_1 \to \mathcal{V}_1$ and $r^{i}_2:\mathcal{X}_2 \to \mathcal{V}_2$, such that, $r^i_1(f_1(x)) = (e^i_1(x),u_i(x))$ and $r^i_2(f_2(x)) = (e^i_2(x),u_i(x))$, where, $I(e^i_j(x);f_k(x))\leq \epsilon$ (for any $i \in \{1,2\}$ and $k\neq j \in \{1,2\}$). Then, there are functions $s_1,s_2$ and $d_1,d_2$, such that, for all $i\neq j \in \{1,2\}$, we have:
\begin{equation}
\mathbb{E}_{x\sim D}[\ell(s_i(u_i(x)),u_j(x))] \leq 1-2^{-\epsilon}
\end{equation}
and also,
\begin{equation}
\mathbb{E}_{x\sim D}[\ell(d_i(e^i_1(x)),e^j_1(x))] \leq 1-2^{-\epsilon}
\end{equation}
In particular, if $\epsilon=0$, $s_1(u_1(x)) = u_2(x)$, $s_1$ is invertible, such that $s^{-1}_1 = s_2$ and $d_1(e^1_1(x)) = e^2_1(x)$, $d_1$ is invertible and $d^{-1}_1 = d_2$.
\end{lemma}

\begin{proof} First, we would like to show that $I(e^1_1(x);u_2(x)) \leq \epsilon$. Assume by contradiction that this is not the case. We consider that, $u_2(x)$ can be represented as a function of $f_2(x)$, since $u_2(x)$ consists of the last coordinate of $r^2_2(f_2(x))$. Therefore, since $r^2_2$ is invertible,
\begin{equation}
\begin{aligned}
I(e^1_1(x);f_2(x)) &= I(e^1_1(x);r^2_2(f_2(x))) \\
&\geq I(e^1_1(x);u_2(x)) > \epsilon
\end{aligned}
\end{equation}
In contradiction to the assumption that $I(e^1_1(x); f_2(x))\leq \epsilon$. By the same argument, we also have, $I(e^2_1(x);u_1(x)) \leq \epsilon$. By the chain rule property of mutual information,
\begin{equation}
\begin{aligned}
I(e^1_1(x),u_1(x);u_2(x)) =& I(u_1(x);u_2(x)) \\
&+ I(e^1_1(x);u_2(x)|u_1(x))\\
\leq& I(u_1(x);u_2(x)) \\
&+ I(e^1_1(x);u_2(x))\\
\end{aligned}
\end{equation}
Therefore, since $r^1_1$ and $r^2_1$ are invertible functions,
\begin{equation}
\begin{aligned}
I(u_1(x);u_2(x)) \geq& I(e^1_1(x),u_1(x);u_2(x)) \\
&- I(e^1_1(x);u_2(x))\\ 
\geq& I(e^1_1(x),u_1(x);u_2(x)) - \epsilon \\ 
=& I(r^1_1(e^1_1(x),u_1(x));u_2(x)) - \epsilon \\
=& I(f_1(x);u_2(x)) - \epsilon \\
=& I(r^2_1(f_1(x));u_2(x)) - \epsilon \\
=& I(e^2_1(x),u_2(x);u_2(x)) - \epsilon \\
\geq& I(u_2(x);u_2(x)) - \epsilon \\
=& H(u_2(x)) - \epsilon
\end{aligned}
\end{equation}
Again, by the chain rule property of mutual information,
\begin{equation}
\begin{aligned}
I(e^1_1(x),u_1(x);e^2_1(x)) =& I(e^1_1(x);e^2_1(x)) \\
&+ I(e^2_1(x);u_1(x)|e^1_1(x))\\
\leq& I(e^2_1(x);e^1_1(x)) \\
&+ I(e^1_1(x);u_1(x))\\
\end{aligned}
\end{equation}
Therefore, since $r^1_1$ and $r^2_1$ are invertible functions,
\begin{equation}
\begin{aligned}
I(e^1_1(x);e^2_1(x)) \geq& I(e^1_1(x),u_1(x);e^2_1(x)) \\
&- I(e^1_1(x);u_1(x))\\ 
\geq& I(e^1_1(x),u_1(x);e^2_1(x)) - \epsilon \\ 
=& I(r^1_1(e^1_1(x),u_1(x));e^2_1(x)) - \epsilon \\
=& I(f_1(x);e^2_1(x)) - \epsilon \\
=& I(r^2_1(f_1(x));e^2_1(x)) - \epsilon \\
=& I(e^2_1(x),u_2(x);e^2_1(x)) - \epsilon \\
\geq& I(e^2_1(x);e^2_1(x)) - \epsilon \\
=& H(e^2_1(x)) - \epsilon
\end{aligned}
\end{equation}
Thus, we conclude that $I(u_1(x);u_2(x)) \geq H(u_2(x)) - \epsilon$ and that $I(e^1_1(x);e^2_1(x)) \geq H(e^2_1(x)) - \epsilon$. In a similar manner, we can show the other directions as well, $I(u_1(x);u_2(x)) \geq H(u_1(x)) - \epsilon$ and $I(e^1_1(x);e^2_1(x)) \geq H(e^1_1(x)) - \epsilon$. Therefore, by Lem.~\ref{lem:fanoconverse} in the Appendix, there are functions $s_1,s_2$ and $d_1,d_2$, such that, for all $i\neq j \in \{1,2\}$, we have:
\begin{equation}
\begin{aligned}
&\mathbb{E}_{x\sim D}[\ell(s_i(u_i(x)),u_j(x))]\\ =&\mathbb{P}_{u_1(x),u_2(x)}[s_i(u_i(x))=u_j(x)] \\
\leq& 1-2^{I(u_1(x);u_2(x))-H(u_j(x))}\\
\leq& 1-2^{H(u_j(x))-\epsilon - H(u_j(x))}= 1-2^{-\epsilon}
\end{aligned}
\end{equation}
and also,
\begin{equation}
\mathbb{E}_{x\sim D}[\ell(d_i(e^i_1(x)),e^j_1(x))] \leq 1-2^{-\epsilon}
\end{equation}
Finally, if $\epsilon=0$, for every $x \in \mathcal{X}$, we have: $s_1(u_1(x))=u_2(x)$, $s_2(u_2(x))=u_1(x)$, $d_1(e^1_1(x))=e^1_1(x)$ and $d_2(e^2_1(x))=e^2_1(x)$. Therefore, $s^{-1}_1 = s_2$ and $d^{-1}_1 = d_2$.
\end{proof}


\end{document}

%% file: illustration.tex
\begin{figure}
\begin{tikzpicture}[auto,scale=1.1]
    \node (empty) {};

    \node (in) at ([shift={(-3,0)}] empty)   {$x \sim D$};
    \node (y) at ([shift={(3.5,1.3)}] empty) {$y(x)$};
    \node (f) at ([shift={(0.5,0)}] empty)   {$f(x)$};
    \node (g) at ([shift={(1.8,-1.5)}] empty) {$g(x,h(x))$};
    \node (h) at ([shift={(3.5,0)}] empty) {$h(x)$};


    \draw[->] (f) -- (h) node[midway] {$c$};


    \draw[->] (in) -- (f) node[midway] {$f$};
    
    \path [l] (h.south) -- ++(0,-1.2)  -- node[pos=-0.5] {} (g.east) ;
    
    \path [l] (in.north) -- ++(0,1.10)  -- node[pos=0.5] {$y$} (y.west) ;

    \path [l] (in.south) -- ++(0,-1.26)  -- node[pos=0.5] {} (g.west) ;
    

\end{tikzpicture}

\caption{The main components of our framework. The EF $g$ is a function of input $x$ and the models' label $h(x)$, which approximates the target function $y$. $h$ is a composition of some representation $f$ and a classifier $c$. Note that $g$ should generate explanations for a specific $h$ and is not generic.}
\label{fig:illustration}
\end{figure}
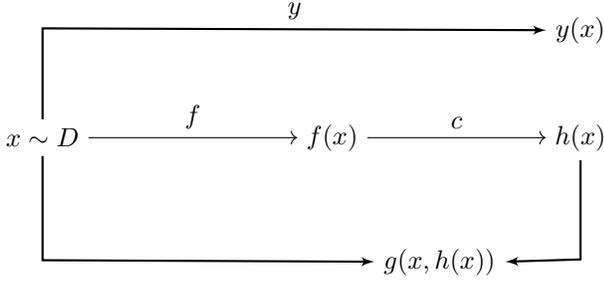